\documentclass[10pt,letterpaper]{article}

\usepackage{cogsci}
\usepackage{pslatex}
\usepackage{apacite}
\usepackage{tikz}

\usetikzlibrary{positioning,arrows,decorations.pathreplacing,shapes}
\usepackage{amssymb}
\usepackage{mathtools}
\usepackage{amsmath}
\usepackage{float}

\usepackage{times}
\usepackage{enumerate}
\newtheorem{theorem}{Theorem}

\newtheorem{proof}{Proof}

\title{Modeling phonological similarity effects on the self-organization of vocabularies}

\author{{\large \bf Javier Vera (jxvera@gmail.com)} \\
  Facultad de Ingenier\'ia y Ciencias, Universidad Adolfo Ib\'a\~{n}ez\\
  Santiago, Chile}

\begin{document}

\maketitle

\begin{abstract}
This work develops a computational model (by Automata Networks) of phonological similarity effects involved in the formation of word-meaning associations on artificial populations of speakers. Classical studies show that in recalling experiments memory performance was impaired for phonologically similar words versus dissimilar ones. Here, the individuals confound phonologically similar words according to a predefined parameter. The main hypothesis is that there is a critical range of the parameter, and with this, of working-memory mechanisms, which implies drastic changes in the final consensus of the entire population. Theoretical results present proofs of convergence for a particular case of the model within a worst-case complexity framework. Computer simulations describe the evolution of an energy function that measures the amount of local agreement between individuals. The main finding is the appearance of sudden changes in the energy function at critical parameters. 

\end{abstract}

\section{Introduction}

In natural language a typical linguistic interaction is not a simple sequence of individual actions, but also a form of \textit{joint activity} that involves cooperation and coordination between participants \cite{tomasello2008}. What is more, inside the dialogue language users tend to converge in their choice of constructions. This mutual convergence process, or \textit{alignment}, has been extensively studied within computational studies of the formation of language through the Naming Game \cite{Steels95,steels2011REVIEW,baronchelli_naming_jstat,1742-5468-2011-04-P04006}. It attempts to ask how on a population of agents only from local interactions it arises a shared word-meaning association (the simplest version of a \textit{vocabulary}). This model considers a finite population of agents, where each one is endowed with a memory in which it stores an in principle unlimited number of words. At each discrete time step a pair of agents is selected: one plays the role of hearer, one plays the role of hearer. First, the speaker refers to an object by using a word. Next, the hearer tries to identify the referent. For this purpose, the hearer inspects its own memory: (1) if the word belongs to the memory of the hearer, both speaker and hearer cancel all the words in their memories, except such word; or (2) the hearer adds the word to its memory, if the word does no belong to its memory. A pairwise interaction such as in the Naming Game only is possible if the interaction is a kind of joint activity. In a more realistic scenario both agents interact in a context where the object is located and share focus on the object by means of pointing or eye-gazing. The interaction continues until both agents reach a common word associated to the object, a linguistic convention. 

Here, it is assumed that the development of linguistic conventions is founded on self-organization mechanisms arising only from \textit{local} interactions between agents \cite{Steels95,Steels96,baronchelli_naming_jstat}. Given this self-organized nature Automata Networks (AN) \cite{Neumann:66,wolfram02} provide the adequate framework to explore alignment from a computational (and mathematical) point of view. AN are extremely simple models where each vertex of a network evolves following a local rule based on the states of ``nearby" vertices. Despite of the simplicity of the defining rules, AN exhibit astonishing rich patterns of behavior. A \textit{linguistic convention} can be considered as a complex pattern.

What is the relationship between cognitive mechanisms of the individuals and the emergence of language on artificial populations? The question entails the study of computational machines which exhibit limitations and constraints of \textit{human} language users. This work stresses a novel question related to models of the self-organization of language: To what extent do working memory constraints influence the alignment of shared conventions on artificial populations of agents? Language users (and therefore agents playing computational games of the formation of language) suffer limited working memory constraints \cite{67167d47d9924cbb86a53b6d61491ed4,baddeley2007working}. Particularly, \textit{phonological similarity} effects suppose that individuals confound word items sharing large portions of phonological content. A classical work \cite{doi:10.1080/14640746608400055} reports experiments where subjects heard sequences of unrelated words and tried to recall in the correct order. The results suggest that memory performance was impaired for phonologically similar words (man, cad, mat, cap, can) versus dissimilar ones (pen, sup, cow, day, hot). This effect is a strong evidence for the existence of the \textit{phonological loop}, understood as part of the short-term working memory system \cite{67167d47d9924cbb86a53b6d61491ed4,baddeley2007working}. 

This paper does not attempt to develop a ``realistic" model, but rather an abstract symbolic approach that extracts the essential elements of the problem. The simulations described here are based on a parameter that measures the amount of phonological confusion between words (considered as a way to describe the influence of working memory limits). The work explores the hypothesis of there being a critical range of the parameter that implies drastic changes in the shared conventions of the entire population. Therefore, the features of language (in particular, the consensus on linguistic conventions) emerge abruptly at some critical range of phonological confusion  \cite{hauser1996evolution}. This hypothesis is strongly related to previous work on on the absence of stages in the formation and evolution of human languages \cite{bickerton98,calvin2001lingua,FS03}. 

The work is organized as follows. Section ``Model" explains basic notions on AN, the instrumentalization of phonological similarity and the rules of interaction. Section ``Theoretical results" reports simple mathematical results related to the convergence of particular cases of the model. The next section (``Simulations") describes simulation tasks based on an energy operator, over a parameter that measures the amount of similarity confusion between words. A brief discussion of the results is presented in the final section. 

\section{Model}

\subsection{Elements of the AN model}
 
Roughly speaking the AN model involves the following elements: 

\begin{enumerate}
\item A regular grid graph: vertices represent individuals; edges represent possible communicative interactions. More generally, as in the section ``Theoretical results", the graph can be a connected, undirected and simple network.

\item Each individual is associated to a \textit{state} that eventually changes along time. This \textit{state} is a way to represent the language of the individual at some time frame. 

\item A set of \textit{local rules} that define how the system changes. The local rule associated to one individual considers as inputs the states of the nearby individuals (the \textit{neighbors}).

\item A function, the \textit{updating scheme}, that indicates the order in which the individuals are updated. Two updating schemes are considered in this work: (1) the \textit{fully-asynchronous} scheme, where at each time step one individual is choosen uniformly at random; and (2) the \textit{sequential} scheme, defined as a permutation of the set of vertices. 
\end{enumerate}

In what follows some of the previous elements will be treated in greater depth. 

\subsection{Basic notions}

The set $P=\{1,...,n\}$ represents the population of individuals, located on the vertices of the regular grid graph $G=(P,E)$, where $E$ is the set of edges. The vertex $u \in P$ only interacts (or ``talks") with the set of adjacent vertices $V_u=\{v \in P: (u,v) \in E \}$ (the \textit{neighborhood}). Each individual considers four neighbors: up, down, left and right ones (\textit{Von Neumann} neighborhood). The individual $u \in P$ is endowed with a \textit{memory} set $M_u$ in which it stores words belonging to a finite set $W$, with $p$ elements. 

Let $\Sigma=\{a_1,...,a_s\}$ be a set of sounds. Each word of $W$ is constructed by a combination of sounds taken from $\Sigma$. For instance, $a_3a_1a_8 \in W$. The \textit{length} of the word $x$ is its number of sounds. For the sake of simplicity, all the words have the same length $L$. $x(k)$, with $k \leqslant L$, denotes the $k-$th sound (or position) of the word $x$. To explicitly measure the amount of phonological similarity between words the \textit{Hamming distance} $H(x,y)$ between the words $x$ and $y$ is defined as the number of positions in which they differ. Consider two words $a_4a_6a_5$ and $a_7a_6a_3$, then $H(a_4a_6a_5,a_7a_6a_3)=2$.

\subsection{Confusion parameter}

To explicitly measure the ability to distinguish between words the \textit{confusion parameter} $\epsilon \in [0,1]$ is defined. Suppose that the vertex $u$ faces two words $x,y$. Then,

\begin{enumerate}
\item[] \textbf{if} $H(x,y) > \epsilon L$, $u$ \textbf{distinguishes} the words $x$ and $y$

\item[] \textbf{else} $u$ \textbf{confounds} the words $x$ and $y$ (or simply ``$x=y$")
\end{enumerate} 

For instance, the individual $u \in P$ is confronted with the words $x=a_1a_8a_6a_4$ and $y=a_1a_8a_6a_3$. Two values of $\epsilon$ are considered, $0$ and $0.5$: 

\begin{itemize}
\item ($\epsilon=0$) $H(x,y) = 1 > \epsilon L = 0 \times 4 = 0$, then $u$ distinguishes the words $x$ and $y$

\item ($\epsilon=0.5$) $H(x,y) = 1 < \epsilon L = 0.5 \times 4=2$, then $u$ confounds the words $x$ and $y$ 
\end{itemize}

\subsection{Local rules}

Inspired in the Naming Game \cite{baronchelli_naming_jstat}, the \textit{local rule} associated to the individual $u \in P$ is based on two possible actions on the memory $M_u$: 

\begin{itemize}
\item $u$ updates its memory $M_u$ by the \textbf{addition} of words; or

\item $u$ \textbf{collapses} its memory if $M_u$ is updated by cancelling all its words, except one of them.
\end{itemize}

Both actions attempt to take into account \textit{lateral inhibition}  strategies \cite{Steels95,steels2011REVIEW} in the alignment process: the individuals add words in order to increase the chance of future agreements (local consensus), and defect the words that do not cooperate with mutual understanding. In a \textit{collapse} the individuals prefer the \textit{minimal} word, according to the \textit{lexicographic order} over the set of words. The \textit{lexicographic order}, denoted $\prec$, is a generalization of the typical alphabetical order of words on the alphabetical order of their component letters (or sounds). For example in the dictionary the word ``Me" appears before ``My" because the letter \textit{e} comes before the letter \textit{y} in the alphabet. In some sense the word ``Me" is \textit{lower} than the word ``My". Formally, the order $<$ is defined on the set $\Sigma$. Two words $x$ and $y$ of length $L$ are considered. Then, $x\prec y$ ($x$ is \textit{lexicographically} lower than $y$) if the first position in which they differ, say $k \leqslant  L$, satisfies $x(k)<y(k)$. For instance given the set of sounds $\Sigma=\{a,b,c,d\}$, with $a<b<c<d$, the words $abc$, $bcd$ and $cda$ satisfy $abc<bcd<cda$. Therefore, $abc$ is the \textit{minimal} word or, in other terms, $\min(\{abc,bcd,cda\})=abc$. Associated to the previous words ``Me" and ``My" it is possible to write $\min(\{\text{``Me"},\text{``My"}\})=\text{``Me"}$. 

The preference for the minimal words can be viewed in accordance with the following hypothetical scenario \cite{Nowak06071999}. It is possible to think in a population of early hominids for which leopards represents a higher risk than cows. So, the word ``leopard" may be more valuable than ``cow". In the terms of this paper ``leopard" can be the \textit{minimal} word.  

At time step $t$ the vertex $u \in P$ is selected according to the updating scheme (\textit{fully asynchronous} or \textit{sequential}). Consider a simple population $P=\{1,2,3,4,5\}$ (each number represents one individual). For a \textit{fully asynchronous} scheme at each time step any individual of $P$ can be selected (for instance, the individual 3). In the next step the individual 3 can be selected again. For a \textit{sequential} scheme a permutation of the set $P$ is defined, for instance, the order 5-4-3-2-1. The individual 5 updates first, then the individual 4 updates, taking into account the effects of the changes in the first individual, and so on. At time step 6 the dynamics starts in the same previous way. In other terms a \textit{sequential} scheme supposes that each individual is updated after that all the other individuals of the population have been updated.

The individual $u \in P$ is completely characterized by its \textit{state} $(M_u,x_u)$, where $M_u$ is the memory to stores words ($M_u$ is a subset of $W$) and $x_u \in M_u$ is a word that $u$ conveys to the vertices of $V_u$. The model induces specific communication roles: the vertex $u$ plays the role of ``hearer" (it receives the words conveyed by its neighbors); the neighbors of $u$ play the role of ``speaker" (they convey words to the vertex $u$). The set of all words conveyed by the speakers can be re-written as two subsets: $N_u$ and $B_u$. Roughly speaking $N_u$  includes the unknown words, and $B_u$ includes the known ones. 

The state pair $(M_u,x_u)$ changes according to the following steps, which define the \textit{local rule} of the automata (see Fig. 1):

\begin{enumerate}
\item[] \textbf{step 1} the vertex $u$ defines two sets:

					\begin{enumerate}
					\item[] $N_u=\{x_v: (v \in V_u) \land (\forall y \in M_u, H(x_v,y) > \epsilon L) \}$
					\item[] $B_u=\{x_v: (v \in V_u) \land (\forall y \in M_u, H(x_v,y) \leqslant \epsilon L) \}$
					\end{enumerate}

\item[] \textbf{step 2} 

				\begin{enumerate}
				\item[] \textbf{if} $N_u \neq \emptyset$, $M_u$ \textbf{adds} the words of $N_u$
				\item[] \textbf{else} $M_u$ \textbf{collapses} in the word $\bar{x}$, selected at random from the set $\{x \in B_u: H(x,\min(B_u)) \leqslant \epsilon L\}$ (that is, the new state is $(\{\bar{x}\},\bar{x})$)
				\end{enumerate}

\end{enumerate}

Step 1 comprises (1) the speaker's behavior (the neighbors convey words to the vertex $u$); and (2) the definition by the hearer of the sets $N_u$ and $B_u$. Given a conveyed word $x_v$, $v \in V_u$, the hearer $u$ decides between: either $x_v$ is added to $N_u$ if for all $y \in M_u, H(x_v,y) > \epsilon L$; or $x_v$ is added to $B_u$, otherwise. 

Step 2 summarizes the behavior of the hearer in order to \textit{align} itself with the speakers. In the case that $N_u \neq \emptyset$, the hearer simply adds to its memory the words of $N_u$. Otherwise ($N_u = \emptyset$), the hearer collapses its memory in the word $\bar{x} \in B_u$. The word $\bar{x}$ is selected uniformly at random from $\{x \in B_u: H(x,\min(B_u)) \leqslant \epsilon L\}$. Thus, the preferred word $\min(B_u)$ can be confused by ``similar" ones (with $H(\cdot,\min(B_u)) \leqslant \epsilon L$).

\begin{figure*}
\begin{center}
\begin{tikzpicture}
  [scale=.45,auto=left,every node/.style={minimum size=0cm}] 
  [thick,scale=0.6,shorten >=0pt]
 
  \node (n1) at (3,5)  {\Large $(\{abcd,bacd\},bacd)$}; 
   \node (n2) at (-0.5,7.5)  {\Large $bacd$};
   \node (n3) at (6.5,7.5)  {\Large  $cabd$};
   \node (n4) at (3,2) {\Large $dabc$};

  \foreach \from/\to in {n1/n2,n1/n3,n1/n4}
    \draw (\from) -- (\to);
    
    
   \node (n5) at (19,5)  {\Large $(\{abcd,bacd,cabd,dabc\},bacd)$}; 
   \node (n6) at (15.5,7.5)  {};
   \node (n7) at (22.5,7.5)  {};
   \node (n8) at (19,2) {};

  \foreach \from/\to in {n5/n6,n5/n7,n5/n8}
    \draw (\from) -- (\to);
    
    
   \node (n9) at (9,5)  {};
   \node (n10) at (12,5)  {};
   
   \path (n9) edge [->, line width=2pt] node[]{\Large \textbf{(addition)}} (n10);
   
   
   \node (n11) at (3,-3)  {\Large $(\{abcd,bacd,cabd\},bacd)$}; 
   \node (n12) at (-0.5,-0.5)  {\Large $bacd$};
   \node (n13) at (6.5,-0.5)  {\Large $cabd$};
   \node (n14) at (3,-6) {\Large $abcd$};

  \foreach \from/\to in {n11/n12,n11/n13,n11/n14}
    \draw (\from) -- (\to);
    
    
   \node (n15) at (19,-3)  {\Large $(\{abcd\},abcd)$}; 
   \node (n16) at (15.5,-0.5)  {};
   \node (n17) at (22.5,-0.5)  {};
   \node (n18) at (19,-6) {};

  \foreach \from/\to in {n15/n16,n15/n17,n15/n18}
    \draw (\from) -- (\to);
    
    
   \node (n19) at (9,-3)  {};
   \node (n20) at (12,-3)  {};
   
   \path (n19) edge [->, line width=2pt] node[]{\Large \textbf{(collapse)}} (n20);

\end{tikzpicture}
\end{center}
\caption{\textbf{Example of the two actions at $\epsilon=0$}. The order $a<b<c<d$ is defined on the set $\Sigma=\{a,b,c,d\}$. $W=\{abcd,bacd,cabd,dabc\}$. The \textit{lexicographic} order establishes $abcd \prec bacd \prec cabd \prec dabc$. The confusion parameter is set to $\epsilon=0$ (that is, the individuals do not confound words). Suppose that at some time step the central vertex ($u \in P$) has been choosen. Four individuals participate in the interaction: the central vertex $u$ and its three neighbors of $V_u$. Associated to each action of the local rule, two different configurations are showed. First row \textbf{(addition)}: $B_u =\{bacd\}$ and $N_u=\{cabd,dabc\}$. Second row \textbf{(collapse)}: $B_u=\{abcd,bacd,cabd\}$ and $N_u=\emptyset$.}
\end{figure*}
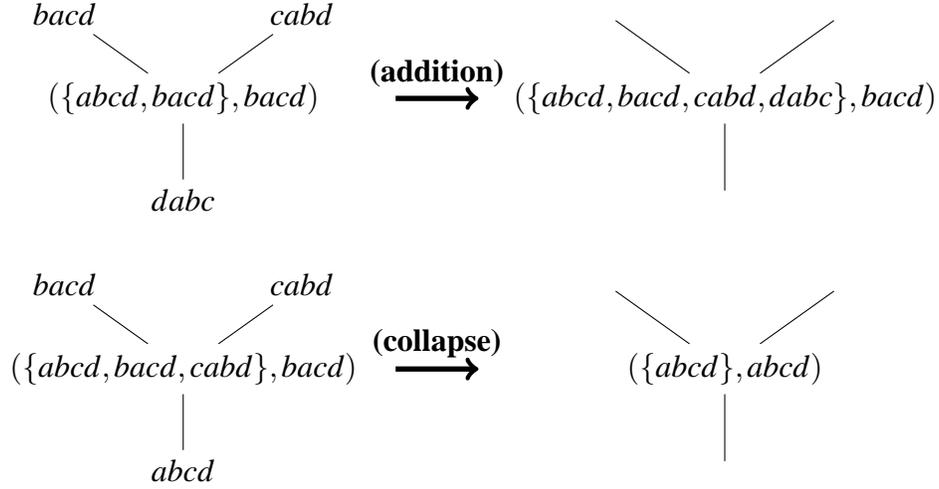

\subsection{Dynamics of the automata}

As initial configuration each individual receives a word constructed by a random combination of $L$ sounds from the set of symbols $\Sigma$. Thus, at $t=0$ each individual is associated to a state of the form $(\{x\},x)$, with $x \in W$. Each discrete time step $t\geqslant 0$ supposes that one individual, say $u$, is selected uniformly at random (the \textit{fully-asynchronous} scheme) or according to a permutation of the set of vertices (the \textit{sequential} scheme). The individual receives the words conveyed by its neighbors: it plays the role of hearer, the neighbors play the role of speaker. Regarding to the conveyed words the individual follows the two defined steps in order to decide possible changes in its own language state:

\begin{itemize}
\item[] \textbf{step 1} the individual defines the sets $N_u$ (unknown words) and $B_u$ (known words). 

\item[] \textbf{step 2} the individual either \textbf{adds} words if $N_u$ is non empty or \textbf{collapses} its memory in the minimal word of $B_u$, if $N_u$ is empty.
\end{itemize}

Both steps involve the possibility of confusion between \textit{similar} words. In the next time step $t+1$ another individual (possibly the same one) is selected, and so on. 

A \textit{fixed point} is a configuration invariant under the application of local rules, which can be interpreted as a final consensus configuration, where all individuals agree about some linguistic convention. 

\section{Theoretical results}

An interesting problem related to the formation of consensus on linguistic conventions is to propose a proof of convergence. Given the mathematical framework of this paper the problem becomes to count (in the \textit{worst} case) the number of simulation steps whom the dynamics needs to stop, that is, to prove that after a finite number of time steps the population reaches a shared linguistic convention. Despite that other works have solved related tasks \cite{DevJTB06} the novelty of the rest of this section is to develop a convergence proof based on the worst-case complexity, which measures the amount of resources (running time) needed by the system if it is considered as an \textit{algorithm}. Running time, defined as the number of steps until the entire population reaches a global consensus language, indicates the largest dynamics performed by the automata given the size $n$ of the population (denoted $\mathcal{O}(f(n))$, where $f(n)$ is a function of $n$). For instance $\mathcal{O}(n^2)$ means that in the worst-case the running time has a growth rate scaling as $n^2$.

In this section individuals are located on a general undirected and connected network (not necessarily a regular \textit{grid}), and they do not confound words, that is, $\epsilon=0$.

It is straightforward to notice that at $\epsilon=0$ $(\forall y \in M_u, H(x_v,y) > \epsilon L)$ is equivalent to $(x_v \notin M_u)$. Roughly speaking, the expression $(\forall y \in M_u, H(x_v,y) > \epsilon L)$ means that $x_v$ is ``different" to every word in $M_u$ and, therefore, $x_v \notin M_u$. As a consequence the two steps of the rule take a simpler form:

\begin{enumerate}
\item[] \textbf{step 1} the vertex $u$ defines two sets:

					\begin{enumerate}
					\item[] $N_u=\{x_v: (v \in V_u) \land (x_v \notin M_u) \}$
					\item[] $B_u=\{x_v: (v \in V_u) \land (x_v \in M_u) \}$
					\end{enumerate}

\item[] \textbf{step 2} 

				\begin{enumerate}
				\item[] \textbf{if} $N_u \neq \emptyset$, $M_u$ \textbf{adds} the words of $N_u$
				\item[] \textbf{else} $M_u$ \textbf{collapses} in the word $\min(B_u)$ (the new state is $(\{\min(B_u)\},\min(B_u))$)
				\end{enumerate}

\end{enumerate}

First of all, a sequential updating scheme is considered. This means that a \textit{permutation} of the vertices is defined. As previously noted, each individual is updated after that all the other individuals of the population have been updated. In the worst case, each time step supposes that one individual adds one word. This process ends when some individual has all possible words, that is, after $p-1$ steps (there are $p$ words). Since the population has size $n$, after $n(p-1)$ steps the individuals must collapse their memories. Then, at step $t^*=n(p-1)$ all individuals have been collapsed at least once. As the theorem shows in detail the minimum conveyed word at $t^*$ propagates in the system until it reaches a fixed point where all individuals convey this minimum word. 

\begin{theorem}

Consider a population of $n$ individuals playing the automata model with the set of $p$ words $W$ and confusion parameter $\epsilon=0$. Then,
\begin{enumerate}[(I)] 

\item for the sequential scheme, the system converges to fixed points in at most $\mathcal{O}(n^2p)$ steps;

\item for the fully-asynchronous scheme, the system converges to a fixed point in expected time $\mathcal{O}(n^2p\log(n))$.

\end{enumerate}
\end{theorem}

\begin{proof}
\normalfont
(I) Initially there are $p$ words. Then, in at most $p-1$ updates a vertex has collapsed for the first time (in the worst case the vertex must add every possible word, one at a time). This implies that in $n(p-1)$ steps ($p-1$ updates of each vertex) all vertices have collapsed at least one time. Let $m$ be the minimum conveyed word at step $t^* = n(p-1)$, and let $u$ be a vertex such that $x_u = m$ (in more precise terms, $m=\min(\{x_u\}_{u \in P})$). 

Since $u$ has collapsed at least one time, the updating scheme is sequential, and $m$ is the minimum word, then $u$ must have another neighbor $v \in V_u$ conveying $m$. In consequence, after $t^*$ both vertices $u$ and $v$ will remain conveying $m$, and at each time a neighbor of any of these two vertices will necessarily collapse in the word $m$. The graph has a diameter of $\mathcal{O}(n)$ and each $np$ steps it occurs a collapse. Therefore, in at most $\mathcal{O}(n^2p)$ steps the system converges to a fixed point where all vertices convey the same word $m$.

(II) In the fully-asynchronous scheme, at each time step a single vertex is picked independently and uniformly at random. The expected convergence time grows by a factor of $O(\log(n))$ with respect to the sequential scheme. Notice that in the proof of the part (I) it is shown that after updating [$\mathcal{O}(np)$ times] [each vertex], a fixed point is reached. From the well known coupon collector's problem (see, for instance, \cite{grimmett2001probability}), the expected number of steps required to update [at least once][every vertex in the graph (or pick every coupon)] is $\mathcal{O}(n\log(n))$. Since in a fully asynchronous updating scheme each step is independent to the others, the result follows. 
\end{proof}

\section{Simulations}

\subsection{Protocol}

To describe the amount of agreement between individuals, an energy operator is defined. This energy-based approach arises from a ``physicist" interpretation, related to \cite{Regnault20094844}: the energy measures the amount of local unstability of the system. At each neighborhood $V_u$ the function $\sum_{v \in V_u} H(x_u,x_v)$ is defined, which measures the \textit{Hamming} distance between the word $x_u$ and the words conveyed by the neighbors of the individual $u$. This function is bounded by $0$ (in case of \textit{agreement}, that is, the individual $u$ and its neighbors convey the same word) and $4L$ (\textit{disagreement}, which means that the individual $u$ and its neighbors convey radically different words). Summing over all individuals defines the total energy function at some time step $t$:

\begin{equation}
E(t)=\frac{1}{4Ln}\sum_{u \in P}\sum_{v \in V_u} H(x_u,x_v)
\end{equation}

The function $E(t)$ is bounded: $0 \leqslant E(t) \leqslant 1$. Thus, the dynamics of the AN model can be understood as the trajectory between initial configurations associated to large amounts of unstability (with $E(0) \sim 1$) and final consensus configurations where $E(t) \sim 0$, or equivalently, all individuals convey similar -in the sense defined by the \textit{Hamming} distance- words.

The analysis focuses on two-dimensional periodic lattices of size $n=128^2$ with Von Neumann neighborhood. The simulations describe $500n$ steps of the energy function $E(t)$. At each time step one vertex (playing the role of hearer) is selected uniformly at random (\textit{fully-asynchronous} scheme). The plots show average values over 20 initial conditions. An initial condition is defined as follows: each individual receives uniformly at random a word constructed by a random combination of $L$ sounds from the set of symbols $\Sigma$, $|\Sigma|=10$. Word-length $L$ varies from: $\{2,4,8,16,32,64\}$. $\epsilon$ is varied from 0 to 1 with an increment of $10\%$. 

Given the worst-case complexity approach to theoretical aspects of convergence it is important to notice that on low-dimensional lattices it seems hard that one individual adds $\mathcal{O}(np)$ words after it collapses, because lattices are low-connected. This intuition suggests that running times on lattices will be lower than theoretical bounds ($\mathcal{O}(n^2p\log(n))$ for the \textit{fully-asynchronous} scheme).

\subsection{Results}

There are several remarkable aspects, as shown in Fig. 2 and Fig. 3. First of all, $E(t=500n)$ versus $\epsilon$ exhibits at $\epsilon=1$ a maximum which is close to 0.5 for all values of $L$ (Fig. 2). Secondly, to the extent $L$ grows, $E(t=500n)$ versus $\epsilon$ evolves more ``smoothly": $L \leqslant 8$ supposes ``ladder" steps which mean that different values of the confusion parameter $\epsilon$ lead to the same energy. Finally, focusing the description on $L=32$ (Fig. 3), it is interesting to notice that
at $\epsilon = 0.7$ the average value of $E(t)$ stops approximately after $t=200n$ steps. This fact exhibits that only a few set of runs does not converge to the global minimum of the $E(t)=0$. This strongly suggests the appearance of a critical parameter $\epsilon^* \sim 0.7$ which clearly defines two phases in the evolution of $E(t=500n)$ versus $\epsilon$: (1) $\epsilon < \epsilon^*$ implies the convergence to the global minimum $E(t)=0$, where all individuals convey the same word; and (2) for $\epsilon \geqslant \epsilon^*$ the dynamics changes drastically until it reaches local minima of the energy function ($E(t) \rightarrow 0.5$).

\begin{figure}[h!]

  \begin{center} 
    
  \includegraphics[scale=0.3]{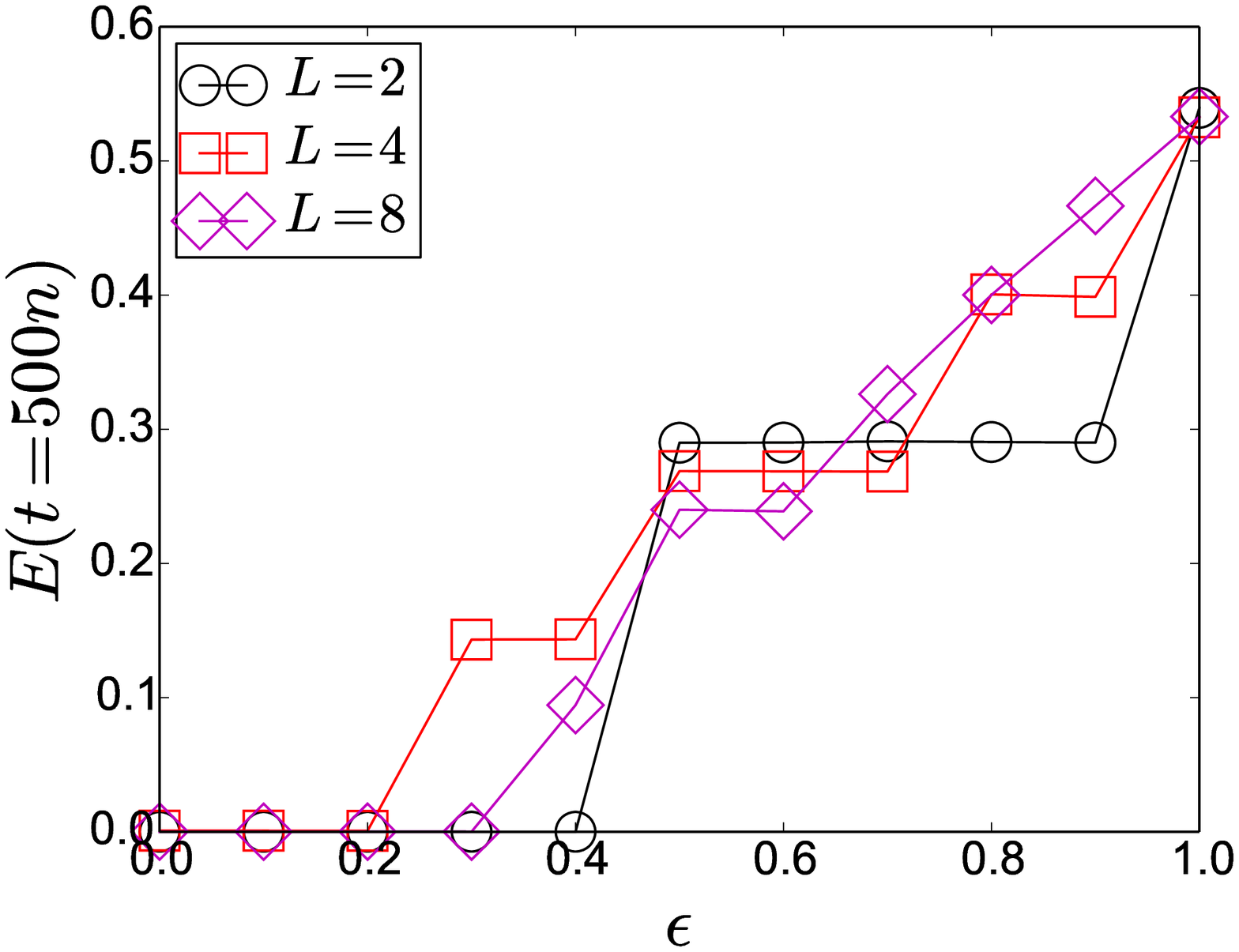}
  \includegraphics[scale=0.3]{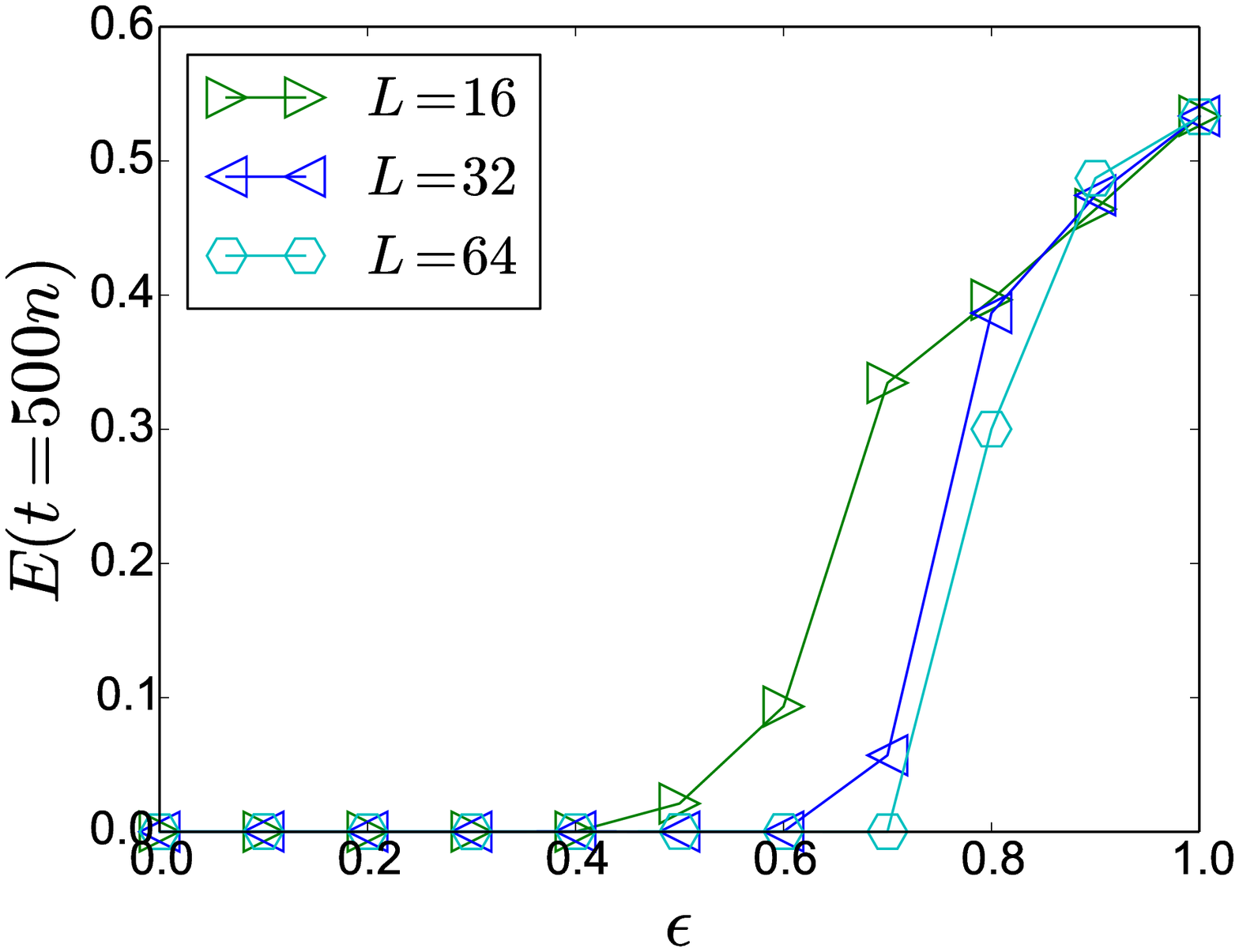}

  \end{center}
  \caption{\textbf{$E(t=500n)$ versus $\epsilon$}. On a two dimensional grid of size $n=128^2$, the figure shows the final value of the energy function versus the parameter $\epsilon$, after $500n$ steps or until they reach the global minimum $E(t)=0$. The plots show averages over 20 initial conditions: for an initial condition, each vertex receives a word constructed by a random combination of $L$ sounds from the set of symbols $\Sigma$, $|\Sigma|=10$. Word-length varies from: $\{2,4,8\}$ (top); and $\{16,32,64\}$ (bottom). $\epsilon$ is varied from 0 to 1 with an increment of $10\%$.}

\end{figure}

\begin{figure}[h!]

  \begin{center} 
    
  \includegraphics[scale=0.3]{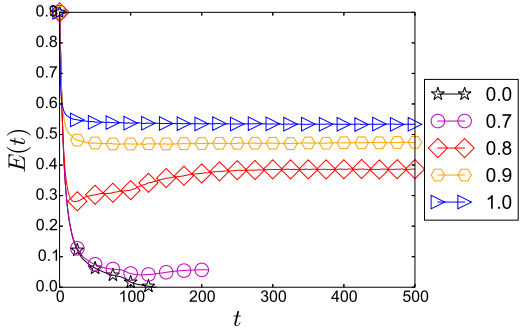}
  
  \end{center}
  \caption{\textbf{$E(t)$ versus $t$, $L=32$, for different values of $\epsilon$}. On a two dimensional grid of size $n=128^2$, the figure shows evolution of the energy function $E(t)$ versus $t$, for different values of the parameter $\epsilon$. Simulations run $500n$ steps or until it reaches the global minimum $E(t)=0$. The plot shows averages over 20 initial conditions: for an initial condition, each vertex receives a word constructed by a random combination of $L$ sounds from the set of symbols $\Sigma$, $|\Sigma|=10$. $L = 32$ and $\epsilon$ is varied from $\{0,0.7,0.8,0.9,1\}$. }

\end{figure}
 
\section{Conclusion}

This work summarizes an AN approach to the formation of linguistic conventions under phonological similarity (and, in general, working memory) mechanisms. The paper presents the evolution of an energy functional, defined as a word ``confusion" average, during the alignment game. Two aspects are remarkable. On the one hand, the appearance of drastic transitions can be related to previous works that focus on the absence of stages in the formation and evolution of human languages (see, for instance, \cite{bickerton98,calvin2001lingua,FS03}). On the other hand, the proposed model becomes an alternative (mathematical) framework for agent-based studies on language formation.

As a first approach to the convergence of the formation of linguistic conventions this paper presents within an AN account simple results of the number of steps needed to reach fixed points. As the main tool the proofs are based on the \textit{worst} case of convergence. 

Many extensions of the proposed model should be studied with the purpose to describe the role of cognitive constraints (and, in particular, the role of (working) memory limits) on the formation of linguistic conventions. First, within a mathematical point of view it seems interesting to explore convergence bounds on regular lattices. Also, a comparison between theoretical and numerical convergence times should be carried out. Second, AN allow to study new aspects of the formation of linguistic conventions. Indeed, the model can be extended to other \textit{updating schemes}, for example, the \textit{synchronous} one, where at each time step all individuals are updated. Third, the results should be compared with larger word lengths ($L$) and several number of symbols ($\Sigma$). Finally, future work should involve more realistic ways to measure the amount of phonological confusion and its effects on the formation of conventions. 

\section*{Acknowledgments}
The author likes to thank CONICYT-Chile under the Doctoral scholarship 21140288.

\bibliographystyle{apacite}

\setlength{\bibleftmargin}{.125in}
\setlength{\bibindent}{-\bibleftmargin}

\bibliography{Vera_04_06_16}

\newcommand{\noop}[1]{}
\begin{thebibliography}{}

\bibitem[\protect\citeauthoryear{%
A.~Baddeley%
}{%
A.~Baddeley%
}{%
{\protect\APACyear{2007}}%
}]{%
baddeley2007working}%
\APACinsertmetastar{%
baddeley2007working}%
Baddeley, A.%
%
\unskip\
\newblock
\APACrefYear{2007}.
\newblock
\APACrefbtitle{Working Memory, Thought, and Action}{Working memory, thought,
  and action}.
\newblock
\APACaddressPublisher{}{Oxford University Press}.
\PrintBackRefs{\CurrentBib}

\bibitem[\protect\citeauthoryear{%
A.~Baddeley%
\ \BBA{} Hitch%
}{%
A.~Baddeley%
\ \BBA{} Hitch%
}{%
{\protect\APACyear{1974}}%
}]{%
67167d47d9924cbb86a53b6d61491ed4}%
\APACinsertmetastar{%
67167d47d9924cbb86a53b6d61491ed4}%
Baddeley, A.%
\BCBT{}\ \BBA{} Hitch, G.%
%
\unskip\
\newblock
\APACrefYearMonthDay{1974}{}{}.
\newblock
\BBOQ{}\APACrefatitle{Working memory}{Working memory}.\BBCQ{}
\newblock
\BIn{} G.~Bower\ (\BED), \APACrefbtitle{Recent Advances in Learning and
  Motivation}{Recent advances in learning and motivation}\ (\BVOL~8,
  \BPG~47-90).
\newblock
\APACaddressPublisher{}{Academic Press}.
\PrintBackRefs{\CurrentBib}

\bibitem[\protect\citeauthoryear{%
A\BPBI D.~Baddeley%
}{%
A\BPBI D.~Baddeley%
}{%
{\protect\APACyear{1966}}%
}]{%
doi:10.1080/14640746608400055}%
\APACinsertmetastar{%
doi:10.1080/14640746608400055}%
Baddeley, A\BPBI D.%
%
\unskip\
\newblock
\APACrefYearMonthDay{1966}{}{}.
\newblock
\BBOQ{}\APACrefatitle{Short-term memory for word sequences as a function of
  acoustic, semantic and formal similarity}{Short-term memory for word
  sequences as a function of acoustic, semantic and formal similarity}.\BBCQ{}
\newblock
\APACjournalVolNumPages{Quarterly Journal of Experimental
  Psychology}{18}{4}{362-365}.
\newblock
\APACrefnote{PMID: 5956080}
\PrintBackRefs{\CurrentBib}

\bibitem[\protect\citeauthoryear{%
Baronchelli%
, Felici%
, Caglioti%
, Loreto%
\BCBL{}\ \BBA{} Steels%
}{%
Baronchelli%
\ \protect\BOthers{.}}{%
{\protect\APACyear{2006}}%
}]{%
baronchelli_naming_jstat}%
\APACinsertmetastar{%
baronchelli_naming_jstat}%
Baronchelli, A.%
, Felici, M.%
, Caglioti, E.%
, Loreto, V.%
\BCBL{}\ \BBA{} Steels, L.%
%
\unskip\
\newblock
\APACrefYearMonthDay{2006}{}{}.
\newblock
\BBOQ{}\APACrefatitle{Sharp Transition towards Shared Vocabularies in
  Multi-Agent Systems}{Sharp transition towards shared vocabularies in
  multi-agent systems}.\BBCQ{}
\newblock
\APACjournalVolNumPages{J. Stat. Mech.}{}{P06014}{}.
\PrintBackRefs{\CurrentBib}

\bibitem[\protect\citeauthoryear{%
Bickerton%
}{%
Bickerton%
}{%
{\protect\APACyear{1998}}%
}]{%
bickerton98}%
\APACinsertmetastar{%
bickerton98}%
Bickerton, D.%
%
\unskip\
\newblock
\APACrefYearMonthDay{1998}{}{}.
\newblock
\BBOQ{}\APACrefatitle{{Catastrophic evolution: The case for a single step from
  protolanguage to full human language}}{{Catastrophic evolution: The case for
  a single step from protolanguage to full human language}}.\BBCQ{}
\newblock
\BIn{} J\BPBI R.~Hurford, S\BPBI M.~Kennedy\BCBL{}\ \BBA{} C.~Knight\ (\BEDS),
  \APACrefbtitle{Approaches to the Evolution of Language: Social and Cognitive
  Bases}{Approaches to the evolution of language: Social and cognitive bases}\
  (\BPGS\ 341--358).
\newblock
\APACaddressPublisher{Cambridge}{Cambridge University Press}.
\PrintBackRefs{\CurrentBib}

\bibitem[\protect\citeauthoryear{%
Calvin%
\ \BBA{} Bickerton%
}{%
Calvin%
\ \BBA{} Bickerton%
}{%
{\protect\APACyear{2001}}%
}]{%
calvin2001lingua}%
\APACinsertmetastar{%
calvin2001lingua}%
Calvin, W.%
\BCBT{}\ \BBA{} Bickerton, D.%
%
\unskip\
\newblock
\APACrefYear{2001}.
\newblock
\APACrefbtitle{Lingua Ex Machina: Reconciling Darwin and Chomsky with the Human
  Brain}{Lingua ex machina: Reconciling darwin and chomsky with the human
  brain}.
\newblock
\APACaddressPublisher{}{MIT Press}.
\PrintBackRefs{\CurrentBib}

\bibitem[\protect\citeauthoryear{%
DeVylder%
\ \BBA{} Tuyls%
}{%
DeVylder%
\ \BBA{} Tuyls%
}{%
{\protect\APACyear{2006}}%
}]{%
DevJTB06}%
\APACinsertmetastar{%
DevJTB06}%
DeVylder, B.%
\BCBT{}\ \BBA{} Tuyls, K.%
%
\unskip\
\newblock
\APACrefYearMonthDay{2006}{}{}.
\newblock
\BBOQ{}\APACrefatitle{How to Reach Linguistic Consensus: A Proof of Convergence
  for the Naming Game}{How to reach linguistic consensus: A proof of
  convergence for the naming game}.\BBCQ{}
\newblock
\APACjournalVolNumPages{The Journal of Theoretical
  Biology}{242(4)}{}{818--831}.
\PrintBackRefs{\CurrentBib}

\bibitem[\protect\citeauthoryear{%
Ferrer-i{-}Cancho%
\ \BBA{} Sol\'{e}%
}{%
Ferrer-i{-}Cancho%
\ \BBA{} Sol\'{e}%
}{%
{\protect\APACyear{2003}}%
}]{%
FS03}%
\APACinsertmetastar{%
FS03}%
Ferrer-i{-}Cancho, R.%
\BCBT{}\ \BBA{} Sol\'{e}, R\BPBI V.%
%
\unskip\
\newblock
\APACrefYearMonthDay{2003}{}{}.
\newblock
\BBOQ{}\APACrefatitle{{Least Effort and the Origins of Scaling in the Human
  Language}}{{Least Effort and the Origins of Scaling in the Human
  Language}}.\BBCQ{}
\newblock
\APACjournalVolNumPages{Proceedings of the National Academy of Science
  (USA)}{100}{}{788--791}.
\PrintBackRefs{\CurrentBib}

\bibitem[\protect\citeauthoryear{%
Grimmett%
\ \BBA{} Stirzaker%
}{%
Grimmett%
\ \BBA{} Stirzaker%
}{%
{\protect\APACyear{2001}}%
}]{%
grimmett2001probability}%
\APACinsertmetastar{%
grimmett2001probability}%
Grimmett, G.%
\BCBT{}\ \BBA{} Stirzaker, D.%
%
\unskip\
\newblock
\APACrefYear{2001}.
\newblock
\APACrefbtitle{Probability and Random Processes}{Probability and random
  processes}.
\newblock
\APACaddressPublisher{}{OUP Oxford}.
\PrintBackRefs{\CurrentBib}

\bibitem[\protect\citeauthoryear{%
Hauser%
}{%
Hauser%
}{%
{\protect\APACyear{1996}}%
}]{%
hauser1996evolution}%
\APACinsertmetastar{%
hauser1996evolution}%
Hauser, M.%
%
\unskip\
\newblock
\APACrefYear{1996}.
\newblock
\APACrefbtitle{The Evolution of Communication}{The evolution of communication}.
\newblock
\APACaddressPublisher{}{MIT Press}.
\PrintBackRefs{\CurrentBib}

\bibitem[\protect\citeauthoryear{%
Loreto%
, Baronchelli%
, Mukherjee%
, Puglisi%
\BCBL{}\ \BBA{} Tria%
}{%
Loreto%
\ \protect\BOthers{.}}{%
{\protect\APACyear{2011}}%
}]{%
1742-5468-2011-04-P04006}%
\APACinsertmetastar{%
1742-5468-2011-04-P04006}%
Loreto, V.%
, Baronchelli, A.%
, Mukherjee, A.%
, Puglisi, A.%
\BCBL{}\ \BBA{} Tria, F.%
%
\unskip\
\newblock
\APACrefYearMonthDay{2011}{}{}.
\newblock
\BBOQ{}\APACrefatitle{Statistical physics of language dynamics}{Statistical
  physics of language dynamics}.\BBCQ{}
\newblock
\APACjournalVolNumPages{Journal of Statistical Mechanics: Theory and
  Experiment}{2011}{04}{P04006}.
\PrintBackRefs{\CurrentBib}

\bibitem[\protect\citeauthoryear{%
Neumann%
}{%
Neumann%
}{%
{\protect\APACyear{1966}}%
}]{%
Neumann:66}%
\APACinsertmetastar{%
Neumann:66}%
Neumann, J. von.%
%
\unskip\
\newblock
\APACrefYear{1966}.
\newblock
\APACrefbtitle{Theory of Self-Reproducing Automata}{Theory of self-reproducing
  automata}.
\newblock
\APACaddressPublisher{Champain, IL}{University of Illinois Press}.
\PrintBackRefs{\CurrentBib}

\bibitem[\protect\citeauthoryear{%
Nowak%
\ \BBA{} Krakauer%
}{%
Nowak%
\ \BBA{} Krakauer%
}{%
{\protect\APACyear{1999}}%
}]{%
Nowak06071999}%
\APACinsertmetastar{%
Nowak06071999}%
Nowak, M\BPBI A.%
\BCBT{}\ \BBA{} Krakauer, D\BPBI C.%
%
\unskip\
\newblock
\APACrefYearMonthDay{1999}{}{}.
\newblock
\BBOQ{}\APACrefatitle{The evolution of language}{The evolution of
  language}.\BBCQ{}
\newblock
\APACjournalVolNumPages{Proceedings of the National Academy of
  Sciences}{96}{14}{8028-8033}.
\PrintBackRefs{\CurrentBib}

\bibitem[\protect\citeauthoryear{%
Regnault%
, Schabanel%
\BCBL{}\ \BBA{} Thierry%
}{%
Regnault%
\ \protect\BOthers{.}}{%
{\protect\APACyear{2009}}%
}]{%
Regnault20094844}%
\APACinsertmetastar{%
Regnault20094844}%
Regnault, D.%
, Schabanel, N.%
\BCBL{}\ \BBA{} Thierry, E.%
%
\unskip\
\newblock
\APACrefYearMonthDay{2009}{}{}.
\newblock
\BBOQ{}\APACrefatitle{Progresses in the analysis of stochastic 2D cellular
  automata: A study of asynchronous 2D minority}{Progresses in the analysis of
  stochastic 2d cellular automata: A study of asynchronous 2d minority}.\BBCQ{}
\newblock
\APACjournalVolNumPages{Theoretical Computer Science}{410}{47–49}{4844 -
  4855}.
\PrintBackRefs{\CurrentBib}

\bibitem[\protect\citeauthoryear{%
Steels%
}{%
Steels%
}{%
{\protect\APACyear{1995}}%
}]{%
Steels95}%
\APACinsertmetastar{%
Steels95}%
Steels, L.%
%
\unskip\
\newblock
\APACrefYearMonthDay{1995}{}{}.
\newblock
\BBOQ{}\APACrefatitle{A Self-Organizing Spatial Vocabulary.}{A self-organizing
  spatial vocabulary.}\BBCQ{}
\newblock
\APACjournalVolNumPages{Artificial Life}{2}{3}{319-332}.
\PrintBackRefs{\CurrentBib}

\bibitem[\protect\citeauthoryear{%
Steels%
}{%
Steels%
}{%
{\protect\APACyear{1996}}%
}]{%
Steels96}%
\APACinsertmetastar{%
Steels96}%
Steels, L.%
%
\unskip\
\newblock
\APACrefYearMonthDay{1996}{}{}.
\newblock
\BBOQ{}\APACrefatitle{Self-organizing vocabularies}{Self-organizing
  vocabularies}.\BBCQ{}
\newblock
\BIn{} \APACrefbtitle{Proceedings of Artificial Life V, Nara,
  Japan}{Proceedings of artificial life v, nara, japan}\ (\BPGS\ 179--184).
\newblock
\APACaddressPublisher{Nara, Japan}{}.
\PrintBackRefs{\CurrentBib}

\bibitem[\protect\citeauthoryear{%
Steels%
}{%
Steels%
}{%
{\protect\APACyear{2011}}%
}]{%
steels2011REVIEW}%
\APACinsertmetastar{%
steels2011REVIEW}%
Steels, L.%
%
\unskip\
\newblock
\APACrefYearMonthDay{2011}{December}{}.
\newblock
\BBOQ{}\APACrefatitle{Modeling the cultural evolution of language}{Modeling the
  cultural evolution of language}.\BBCQ{}
\newblock
\APACjournalVolNumPages{Physics of Life Reviews}{8}{4}{339-356}.
\PrintBackRefs{\CurrentBib}

\bibitem[\protect\citeauthoryear{%
Tomasello%
}{%
Tomasello%
}{%
{\protect\APACyear{2008}}%
}]{%
tomasello2008}%
\APACinsertmetastar{%
tomasello2008}%
Tomasello, M.%
%
\unskip\
\newblock
\APACrefYear{2008}.
\newblock
\APACrefbtitle{The Origins of Human Communication}{The origins of human
  communication}.
\newblock
\APACaddressPublisher{}{MIT Press}.
\PrintBackRefs{\CurrentBib}

\bibitem[\protect\citeauthoryear{%
Wolfram%
}{%
Wolfram%
}{%
{\protect\APACyear{2002}}%
}]{%
wolfram02}%
\APACinsertmetastar{%
wolfram02}%
Wolfram, S.%
%
\unskip\
\newblock
\APACrefYear{2002}.
\newblock
\APACrefbtitle{A New Kind of Science}{A new kind of science}.
\newblock
\APACaddressPublisher{}{Wolfram Media}.
\PrintBackRefs{\CurrentBib}

\end{thebibliography}

\end{document}